%% file: main.tex
\documentclass{article}
\usepackage[utf8]{inputenc}
\usepackage{graphicx} 
\usepackage{amsmath}
\usepackage{graphicx}
\usepackage{scalerel}
\usepackage[margin=1in]{geometry}
\title{The Sample Complexity of Gradient Descent \\ in Stochastic Convex Optimization}
\author{%
  Roi Livni\\
  School of Electrical Engineering\\
  Tel Aviv University\\
  \texttt{rlivni@tauex.tau.ac.il} \\
}
\input{header}

\renewcommand{\S}{\mathcal{S}}
\renewcommand{\SS}{\mathbf{S}}
\begin{document}

\maketitle

\begin{abstract}
    We analyze the sample complexity of full-batch Gradient Descent (GD) in the setup of non-smooth Stochastic Convex Optimization. We show that the generalization error of GD, with  common choice of hyper-parameters, can be $\tilde \Theta(d/m + 1/\sqrt{m})$, where $d$ is the dimension and $m$ is the sample size. 
    This matches the sample complexity of \emph{worst-case} empirical risk minimizers. That means that, in contrast with other algorithms, GD has no advantage over naive ERMs. 
    Our bound follows from a new generalization bound that depends on both the dimension as well as the learning rate and number of iterations. Our bound also shows that, for general hyper-parameters, when the dimension is strictly larger than number of samples, $T=\Omega(1/\epsilon^4)$ iterations are necessary to avoid overfitting. This resolves an open problem by \citet*{schliserman2024dimension, amir2021sgd},  and improves over previous lower bounds that demonstrated that the sample size must be at least square root of the dimension.

\end{abstract}
\newpage
\section{Introduction}
Stochastic Convex Optimization (SCO) is a theoretical model that depicts a learner that minimizes a (Lipschitz) convex function, given finite noisy observations of the objective \cite{shalev2009stochastic}. While often considered simplistic, in recent years SCO has become a focus of theoretical research, partly, because of its importance to the study of first-order optimization methods.
But, also, it has become focus of study because it is one of few theoretical settings that exhibit \emph{overparameterized learning} . 
In more detail, classical learning theory often focuses on the tension between number of samples, or training data, and the complexity of the model to be learnt. A common wisdom of classical theories \cite{alon1997scale, vapnik2015uniform, blumer1989learnability, haussler2018probably} is that, to avoid overfitting, the complexity of a model should be adjusted in proportion to the amount of training data. However, recent advances in Machine Learning have challenged this viewpoint. Evidently \cite{neyshabur2014search,zhang2021understanding}, state-of-the-art algorithms generalize well but without, explicitly, controlling the capacity of the model to be learnt. In turn, today, it is one of the most emerging challenges, for learning theory, to understand learnability when the number of parameters in a learnt model exceeds the number of examples, and when, seemingly, nothing withholds the algorithm from overfitting.

Towards this, we look at SCO. In SCO, \citet*{shalev2009stochastic} showed how algorithms can overfit with \emph{dimension dependent} sample size. But, at the same time, it was known \cite{bousquet2002stability, zinkevich2003online} that there are algorithms that provably avoid overfitting with far fewer examples than dimensions. As such, SCO became a canonical model to study how a well-designed algorithm can avoid overfitting even when the number of examples is too small to guarantee generalization by an algorithmic-independent argument \cite{amir2021never,amir2021sgd, sekhari2021sgd, feldman2016generalization, shalev2009stochastic, attias2024information, livni2024information, schliserman2024dimension, amir2022thinking, carmon2023sample}.
A step towards understanding \emph{what} induces generalization is to identify \emph{which} algorithms generalize. Then, we can ask what yields the separation. Surprisingly, for many well-studied algorithms this question is not always answered.
%
Perhaps the simplest algorithm, whose sample complexity is not yet understood, is Gradient Descent (GD). And we turn to the basic question of the sample complexity of gradient descent.

While this question remained open, there have been several advancements and intermediate answers: The first, dimension independent, generalization bound was given by \citet*{bassily2020stability} that provided stability bounds \cite{bousquet2002stability}. The result of \citeauthor{bassily2020stability} demonstrated that, GD can have \emph{dimension-independent} sample complexity rate. However, to achieve that, one has to use non-standard choice of hyperparameters which affects the efficiency of the algorithm. In particular, the number of rounds becomes quadratic in the size of the sample (as opposed to linear, with standard choice). On the other hand, a classical covering argument shows that linear dependence in the dimension is the worst possible, for any empirical risk minimizer, irrespective of properties such as stability. 

In terms of lower bounds, \citet*{amir2021sgd} were the first to show that GD may have a dimension dependence in the sample complexity. They showed that, with natural hyperparameters, the algorithm must observe number of samples that is at least logarithmic in the dimension. This result was recently improved by \citet*{schliserman2024dimension} that showed that at least square root of the dimension is required. Taken together, so far it was shown that either the algorithm's hyperparameters are tuned to achieve stability, at a cost in running time, or the algorithm must suffer \emph{some} dimension dependence, linear at worst square root at best.

Here, we close the gap and show that linear dependence is necessary. Informally, we provide the following generalization error bound, in terms of dimension $d$, sample size, $m$, and hyperparameters of the algorithm, $\eta$ and $T$ (the learning rate and number of iterations). We show that when $T$ is at most cubic in the dimension (see \cref{thm:main} for a formal statement):

\[ \textrm{Generalization gap of GD} = \Omega\left ( \min\left\{ \frac{d}{m},1\right\}\cdot \min \left\{\eta \sqrt{T},1\right\}\right).\]

The first factor in the RHS describes the linear dependence of the generalization error in the dimension, and corresponds to the optimal sample complexity of empirical risk minimizers, as demonstrated by \citet*{carmon2023sample}.
The second term lower bounds the stability of the algorithm \cite{bassily2020stability}, and played a similar role in previous bounds \cite{schliserman2024dimension, amir2021sgd}. Each factor is optimal for a certain choice of hyper-parameters, and cannot be improved. Most importantly, for a standard choice of $\eta=O(1/\sqrt{T})$, the first term is dominant, and the aformentioned lower bound is complemented with the upper bound of \citet{carmon2023sample}. Our result implies, then, a sample complexity of $\tilde \Theta(d/m+1/\sqrt{m})$. When $d\ge m$, the second factor is dominant. When running time is at most quadratic in number of examples, this term also governs the stability of the algorithm, hence the result of \citet{bassily2020stability} provides a complementary upper bound    (see further discussion in \cref{sec:discussion}).

\section{Background}\label{sec:background}
We consider the standard setup of Stochastic Convex Optimization (SCO) as in \cite{shalev2009stochastic}. Set $\W= \{w: \|w\|\le 1\}$, and let $\Z$ be an arbitrary, finite domain (our main result is a lower bound, hence finiteness of $\Z$ is without loss of generality). We assume that there exists a function $f(w,z)$ that is convex and $L$-Lipschitz in $w\in \W$ for every choice of $z\in \Z$. Recall that a function $f$ is convex and $L$-Lipschitz if for any $w_1,w_2 \in \W$ and $0\le \lambda \le 1$:
\begin{equation} f(\lambda w_1 + (1-\lambda)w_2)\le \lambda f(w_1) + (1-\lambda) f(w_2),~ \textrm{and},~ |f(w_1)-f(w_2)|\le L\|w_1-w_2\|.\end{equation}

\paragraph{First order optimization} Algorithmically we require further assumptions concerning any interaction with the function to be optimized.
 Recall \cite{rockafellar2015convex} that, for fixed $z$, the sub-gradient set of $f(w,z)$ at point $w$ is the set:
\[ \partial f(w,z)= \left\{g: f(w',z) \ge f(w,z) +g^\top (w'-w), \forall w'\in \W\right\}.\]
A first order oracle for $f$ is a mapping $\O_z(w)$ such that
$ \O_z(w)\in \partial f(w,z).$
Our underlying assumption is that a learner has a first order oracle access. In other words, given a function $f(w,z)$, we will assume that there is a procedure $\O_z$ that calculates and returns a subgradient at every $w$ for every $z$. Recall \cite{rockafellar2015convex,bubeck2015convex} that When $|\partial f(w,z)|=1$, the function is differentiable, at $w$, and in that case, the unique subdifferential is the gradient $\nabla f(w,z)$.

\subsection*{Learning}
A learning algorithm $A$, in SCO, is any algorithm that receives as input a sample $S=\{z_1,\ldots, z_m\}\in \Z^m$ of $m$ examples, and outputs a parameter $w_S$. An underlying assumption in learning is that there exists a distribution $D$, unknown to the learner $A$, and that the sample $S$ is drawn i.i.d from $D$. The goal of the learner is to minimize the population loss:
\[ F(w) = \E_{z\sim D}[f(w,z)],\]
More concretely, We will say that the learner has sample complexity $m(\epsilon)$ if, assuming $|S|\ge m(\epsilon)$, then w.p. $1/2$ (Again, because we mostly care about lower bounds, fixing the confidence will not affect the generality of our result):
\begin{equation}\label{eq:learn} F(w_S) - \min_{w\in \W} F(w) \le \epsilon.\end{equation}

\paragraph{Empirical Risk Minimization} A natural approach to perform learning is by \emph{Empirical Risk Minimization} (ERM). Given a sample $S$, the empirical risk is defined to be: \[F_S(w) = \frac{1}{|S|} \sum_{z\in S} f(w,z).\] An $\epsilon$-ERM is any algorithm that, given sample $S$, returns a solution $w_S\in W$ that minimizes the empirical risk up to additive error $\epsilon>0$.
%
Recently, \citet{carmon2023sample} showed that any $\epsilon$-ERM algorithm has a sample complexity bound of \begin{equation}\label{eq:erm} m(\epsilon)= \tilde O\left(\frac{d}{\epsilon}+\frac{1}{\epsilon^2}\right),\end{equation}
The above rate is optimal up to logarithmic factor \cite{feldman2016generalization}. Namely, there exists a construction and an ERM that will fail, w.p. $1/2$, unless $m=\Omega(d/\epsilon)$ examples are provided\footnote{the $\Omega(1/\epsilon^2)$ sample complexity bound is more straightforward and follows from standard information-theoretic arguments}. Importantly, though, there are algorithms that can learn with much smaller sample complexity. In particular SGD \cite{zinkevich2003online}, stable-GD \cite{bassily2020stability} and regularized ERMs \cite{bousquet2002stability}.

\subsection*{Gradient Descent}
We next depict Gradient Descent whose sample complexity is the focus of this work. GD depends on hyperparameters $T\in \mathbb{N}$ and $\eta\ge 0$ and operates as follows on the empirical risk. The algorithm receives as input a sample $S=\{z_1,\ldots, z_m\}$, defines $w_0=0$, operates for $T$ iterations according to the following recursion:
 \begin{equation}\label{eq:gd} w_{t} = \Pi\left[w_{t-1} - \frac{\eta}{|S|}\sum_{z\in S} \O_z(w_t)\right] \Rightarrow~  w_S^{GD} : = \frac{1}{T} \sum_{t=1}^T w_t,
 \end{equation}
 where $\Pi$ is the projection onto the unit ball, and $\O_z(w_t)$ is a subgradient of the loss function $f(w,z)$. The final output, $w_S^{GD}$, of the algorithm is the averaged iterate (our result, though, can be generalized to other possible suffix-averages such as, say, outputting the last iterate, see \cref{thm:main_suffix}).
GD constitutes an $\epsilon$-ERM. Concretely, it is known \cite{bubeck2015convex, nesterov2013introductory} that GD minimizes the empirical risk and its optimization error is given by:
\begin{equation}\label{eq:gd_empirical} F_S(w_S^{GD}) - \min_{w\in W}F_S(w) = \Theta\left(\min\left\{\eta + \frac{1}{\eta T},1\right\}\right).\end{equation} 
The above bound is tight irrespective of the dimension\footnote{For completeness, we demonstrate the lower bound for $d=1$ at \cref{apx:oplow}}. 
The population loss have also been studied, and \citet{bassily2020stability} demonstrated the following learning guarantee:

\begin{equation}\label{eq:gd_generalization} \E_{S\sim D^m}\left[ F(w_S^{GD}) - \min_{w\in W}F(w)\right] = O\left(\frac{1}{\eta T} + \frac{\eta T}{m} + \eta\sqrt{T}\right).\end{equation} 
The last two terms in the RHS follow from a stability argument, provided in \cite{bassily2020stability}, and the first term follows from the optimization error of GD as depicted in \cref{eq:gd_empirical}. 
Notice that there is always an $O(\eta \sqrt{T})$ gap between the generalization error and empirical error of gradient descent. 

\section{Main Result}
\begin{theorem}\label{thm:main}
For every $d\ge 4096, T\ge 10, m\ge 1$ and $\eta>0$, there exists a distribution $D$, and a $4$-Lipschitz convex function $f(w,z)$ in $\mathbb{R}^{d+1}$, such that for \emph{any} first order oracle of $f(w,z)$, with probability $1/2$, if we run GD with $\eta$ as a learning rate then:
 \[ F(\wgd) - F(0) \ge \thebound.\]

\end{theorem}
We remark, that the above theorem is true for \emph{any} suffix averaging (e.g. last iterate), and not restricted to the averaged iterate (see \cref{thm:main_suffix}).
We now specialize our bound for two interesting regimes. First, we improve previous dependence in the dimension in \cite{amir2021sgd, schliserman2024dimension} and obtain a generalization error bound for $d=\Omega(m+T^{1/3})$:
\begin{corollary}\label{cor:main1}
    Fix $\eta$, and suppose $d =\Omega\left( m+ T^{1/3}\right)$. There exists a distribution $D$, and an $O(1)$-Lipschitz convex function $f(w,z)$ in $\mathbb{R}^{d}$, such that for \emph{any} first order oracle of $f(w,z)$, with probability $1/2$, if we run GD for $T$ iterations, then:
    \begin{equation}\label{eq:main1} F(\wgd) - F(0) \ge \Omega\left( \min\left\{\eta \sqrt{T}+ \frac{1}{\eta T},1\right\}\right).\end{equation}
\end{corollary}
The first term follows from \cref{thm:main}, the second term follows from the optimization error in \cref{eq:gd_empirical}. \Cref{eq:main1} does not hold for $d< m$, and the linear improvement over \cite{schliserman2024dimension, amir2021never} is tight. This can be seen from \cref{eq:gd_empirical} that shows that GD achieves $\epsilon$ empirical excess error when $\eta= O(1/\sqrt{T})$ and $T=O(1/\epsilon^2)$. \Cref{eq:main1} becomes vacuous for such choice of parameters, but \citet{carmon2023sample} showed that the sample complexity of \emph{any ERM} is bounded by $\tilde O((d+\sqrt{m})/m)$. However, as depicted next, this upper bound becomes tight and GD does not improve over a worst-case ERM:

\begin{corollary}\label{cor:main2}
    Suppose $T= O(m^{1.5})$, and $\eta = \Theta(1/\sqrt{T})$. There exists a distribution $D$, and an $O(1)$-Lipschitz convex function $f(w,z)$ in $\mathbb{R}^{d}$, such that for any first order oracle of $f$, with probability $1/2$, if we run GD with $\eta$, for $T$ iterations:
    \[ F(\wgd) - F(0) \ge  \lbound{\frac{d}{m}+\frac{1}{\sqrt{m}}}.\]
\end{corollary}
\Cref{cor:main2} complements \citet{carmon2023sample} upper bound, and improves over \citet{feldman2016generalization} lower bound that only showed existence of \emph{some} ERM with the aforementioned sample complexity. To see that \cref{cor:main2} follows from \cref{thm:main}, notice that when $d\le \sqrt{m}$, then $d/m < 1/\sqrt{m}$ and the bound is dominated by the second term, which is a well known-information theoretic lower bound for learning. When $d> \sqrt{m}$, and $T<m^{1.5}$ we have that $T\le d^3$, plugging $\eta=O(1/\sqrt{T})$ yields the bound.

\subsection{Discussion}\label{sec:discussion}

\Cref{thm:main} provides a new generalization error bound for GD. It shows that the worst case sample complexity for ERMs, derived by \citet{feldman2016generalization}, is in fact applicable also to a very natural first order algorithm, and not, just, to abstract ERMs. This Highlights the importance of choosing the right algorithm for learning in SCO. As discussed, the bound is tight in several regimes, nevertheless still there are unresolved open problems.

\paragraph{Stability in low dimension}
When GD is treated as a naive empirical risk minimizer, and $\eta = O(1/\sqrt{T})$, $T=O(m)$, there is no improvement, using GD, over a worst-case ERM. In the other direction, for dimension that is linear in $m$, one cannot improve over the $\Omega (\eta \sqrt{T})$ term that governs stability. Our bound, though, provide a hope that stability in low dimension can yield an improved bound. In particular, consider the case where $\eta =1/T^{1/4}$ and $d< m$. This is a case where we apply a stable algorithm in small dimensions. Our bound does not negate the possibility of an improved generalization bound. That would mean that, at least at some regime, GD can improve over the worst-case ERM behaviour. We leave it as an open problem for future study
\begin{open}
Is there a generalization bound for GD such that
\[ \E_{S\sim D^m} \left[F(w_S^{GD})- \min_{w\in \W} F(w)\right] = O\left(\frac{d\eta \sqrt{T}}{m}+\frac{1}{\sqrt{m}} \right).\]
Alternatively, can we prove an improved generalization error bound such that:
\[ \E_{S\sim D^m} \left[F(w_S^{GD})- \min_{w\in \W} F(w)\right] = \Omega\left(\min\left\{\frac{d}{m}, \eta\sqrt{T},1\right\}\right).\]
\end{open}
\paragraph{Late stopping}
Another regime where there is a gap between known upper bound and lower bound appears when $T=\Omega(m^2)$. Specifically, the stability upper bound for GD by \citet{bassily2020stability} gives
\[ \E_{S\sim D^m} \left[F(w_S^{GD}) -\min_{w\in \W} F(w)\right] =\ubound{\eta\sqrt{T}+ \frac{\eta T}{m} + \frac{1}{\eta T}}.\]
By \cref{cor:main1}, for large enough dimension:
\[ \E_{S\sim D^m} \left[F(w_S^{GD})- \min_{w\in \W} F(w)\right] = \lbound{\eta\sqrt{T}+  \frac{1}{\eta T}}.\]
When $T=O(m^2)$, the two bounds coincide. Indeed, for the $\eta T/m$ term to dominate the $\eta\sqrt{T}$ term, we must have $T=\Omega(m^2)$. One has to take at least $T=O(m^2)$ iterations in order to generalize with GD (in fact, any full batch method \cite{amir2021never}), however $T=O(m^2)$ iterations are sufficient. Nevertheless, the above gap does yield the possibility of an \emph{unstable} GD method that does generalize. Particularly, if we just regulate the term $\eta \sqrt{T}$, but allow $\eta T/m = \Omega(1)$, then this may yield a regime where GD is unstable (and ERM bounds do not apply) and yet generalize.

\begin{open}
    Are there choices of $\eta$ and $T$ (that depend on $m$) such that $\eta T/m \in \Omega(1)$, but GD has dimension indpendent sample complexity? 
\end{open}
Notice that the $\eta T/m$ term also governs stability in the \emph{smooth} convex optimization setup \cite{hardt2016train}. Recall that a function $f(w,z)$ is said to be $\beta$-smooth if for all $z$, $f(w,z)$ is differentiable, and the gradient is an $\beta$-Lipschitz mapping \cite{hazan2016introduction, bubeck2015convex}. For smooth optimization, even if $\eta \sqrt{T}=\Omega(1)$, GD is still stable. \citet*{hardt2016train} showed that the stability of GD in the smooth case is governed by $O\left(\frac{\eta T}{m}\right)$ for $\eta < 1/\beta$. Therefore, the question of generalization when $\eta T/m \in \Omega(1)$ remains open, even under smoothness assumptions:

\begin{open}
    Assume that $f(w,z)$ is $\Theta(1)$-smooth. What is the sample complexity of GD, when $\eta$ and $T$ are chosen so that 
    $\eta +\frac{1}{\eta T} = o(1),$    
    but $\frac{\eta T}{m}=\Omega(1)$.
\end{open}

\section{Technical Overview}\label{sec:overview}
We next provide a high level overview of our proof technique. For simplicity of exposition we begin with the case $T=m=d$.
We begin by a brief overview of previous construction by \citet{amir2021sgd} that demonstrated \cref{cor:main1} when $m=\Omega(\log d)$. The construction in \cite{amir2021sgd} can be decomposed into three terms:
\[ f(w,z) = g(w,z)+ N_0(w) + h(w,z).\]
The function $g$ has the property that an ERM may fail to learn, unless dimension dependent sample size is considered. \citet{amir2021sgd} incorporated \citet{shalev2009stochastic} construction. Later, \cite{schliserman2024dimension} used Feldman's function \cite{feldman2016generalization} to construct $g$. The shift from the construction depicted in \citet{shalev2009stochastic} to Feldman's function is the first step that allows to move from logarithmic to polynomial dependence in the dimension.
In both constructions an underlying property of $g$ is that there exists a distribution $D$ such that, for small samples, there are overfitting minima. Concretely, there exists a $w_S \in \{0,1/\sqrt{d}\}^d$  such that
\begin{equation}\label{eq:berm} \frac{1}{|S|}\sum_{z\in S} g(w_S,z) - \E_{z\sim D} [g(w_S,z)] = \Omega(1).\end{equation}

The challenge is then, to make gradient descent's trajectory move towards the point $w_S$. The idea can be decomposed into two parts:

\subsection*{Simplifying with an adversarial subgradient:}
To simplify the problem, let us first ease the challenge and suppose we can choose our subgradient oracle in a way that depends on the observed sample. 
Let $N_0$ be the Nemirovski function \cite{bubeck2019complexity}:
\[ N_0(w) = \max\{- w(i),0)\}.\]

Notice that $N_0$ is not differentiable and the choice of subgradient at certain points is not apriori determined. For example, notice that every standard basis vector $-e_i \in \partial N(0)$. More generally, given a sample $S$, let $I=i_1,\ldots, i_{d'}$ be exactly the set of indices such that $w_S$, from \cref{eq:berm}, $w_{S}(i)\ne 0$. Now assume by induction that $w_t(i)> 0$ exactly for $i=i_1,\ldots i_t$, then one can show that we can define the subgradient oracle of $N_0$:
\[ \O(w_t) = -e_{i_{t+1}} \in \partial N_0(w_t).\] 
In that case $w_{t+1}$ will satisfy our assumption for $i_{t+1}$ and we can continue to follow this dynamic for $T$ steps.

Notice that, in this case, GD will converge to $w_S$ (if $\eta=1/\sqrt{d}$ which we assume now for concreteness). One can also show that the output of GD (the averaged iterate) will overfit.
The caveat is that our subgradient oracle depends on the sample $S$. In reality, the sample is drawn independent of the subgradient oracle. and all previous constructions, as well as ours need to handle this. This is discussed in the next section. But before that, let us review another challenge which is when $T\ne d$:

\paragraph{When $d\ll T$}
Another challenge we face with the construction above is that it works when we assume that $T\approx d$. That is because, in Nemirovski's function, the number of iterates we can perform is bounded by the dimension. After $d$ iterations we will end up with the vector $v= \sum_{t=1}^d \eta e_{i_t}$. If $T= \omega(d)$ then $\eta = o(1/\sqrt{d})$,
and the dynamic will end up with a too small norm vector to induce a sizeable population loss. This strategy will provide, at best, with a factor of the form $\Omega\left(\eta\sqrt{\min\{d,T\}}\right)$. Such a factor may be unsatisfactory in a very natural setting where, say, $T=O(m)$, $\eta = O(1/\sqrt{m})$, and $d= \Omega(\sqrt{m})$. To obtain the $d^3$ dependence, we perform the following alternation over the Nemirovski function. Consider the function:

\begin{equation}\label{eq:nemdef1} N(w) = \max\{0,\max_{i\le d}\{-w(i)\},\max_{i\le j\le d}\{w(j)-w(i)\}\}.\end{equation}
And suppose that at each iteration we return a subgradient as follows:

\begin{itemize}
    \item If there is $i\le d$, such that $w(i)=w(i+1)> \eta$, return subgradient $e_{i+1} -e_{i}$ and $w$ is updated by $w_{t+1}= w_t-\eta e_{i+1} + \eta e_{i}$.
    \item If there is no such $i$, then take the minimal $i$ (if exists) such that $w(i)=0$, and return subgradient $-e_i$ and update $w_{t+1}=w_t + \eta e_i$.
    \item When non of the above is met, return subgradient $0$.
\end{itemize}

The dynamic of the above scheme is depicted in \cref{fig:nemonster}, and solves the problem when $T\approx d^3$. One can show that GD will run for at least $d^3\approx T$ iterations, and will increase $O(d)$ coordinates, each, on average, by an order of $O(\eta d )$. This is better than the increase of $\eta$, in each coordinate that we get from Nemirovski's function. In this way we obtain the improved result of $\eta\sqrt{T}$, even when $T \approx d^3$. 

\paragraph{When $T \ll d$,} when the number of iterations is smaller than $d$ we face a different challenge. The immediate solution is to embed in $\mathbb{R}^d$ a construction from $\mathbb{R}^T$, this will provide us with the $\Omega(\eta \sqrt{T})$ term but, on the other hand, such a construction will not yield a $\Omega(d/m)$ term. A different approach, that exploits the dimension to its fullest, is to consider  blocks of coordinates and operate on those instead of single coordinates.

The conclusive outcome incorporates both ideas together, and we replace the Nemirovski function with a version of \cref{eq:nemdef1} that operates on $O(T^{1/3})$ blocks of coordinates. And this concludes our construction. We next move on to the challenge of replacing the data dependent oracle with a standard first order oracle.

\begin{figure}
\centering
\vstretch{1}{\includegraphics[scale=.15]{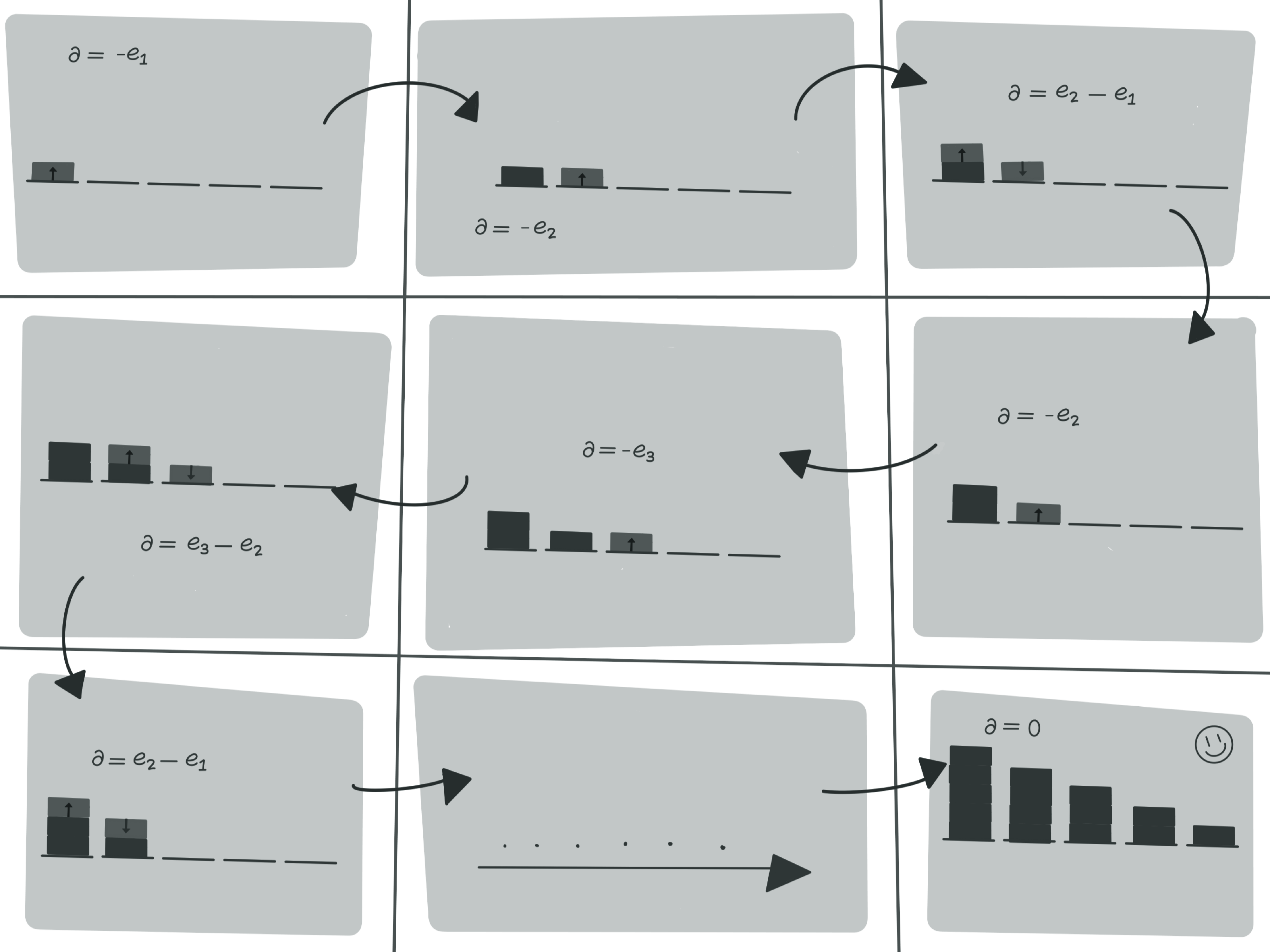}}
    \caption{Depiction of the dynamics induced by \cref{eq:nemonster} and our choice of sub-differentials}
    \label{fig:nemonster}
\end{figure}
\subsection*{Reduction to sample dependent oracle:}
As discussed, the construction above does not yield a lower bound as it relies on a subgradient oracle that is dependent on the whole sample. To avoid such dependence of the oracle on the sample, we observe that if we can infer the sample $S$ from the trajectory, i.e. if the state $w_t$ ``encodes" the sample, then formally the subgradient is allowed to ``decipher" the sample from the point $w_t$. In that way we achieve this behaviour of sample dependent subgradient oracle. This part becomes challenging and may depend on the way we choose $g$, and $N$. The simplest case, studied by \citet{amir2021sgd}, introduced the third function, $h$, which was a small perturbation function that elevates coordinates in $I$ and inhibits coordinates not in $I$. The function $h$ depends on $z$ and not on $S$, hence it cannot know apriori $I$. But, an important observation is that, in \citet{shalev2009stochastic} construction, if $i\notin I$, there exists $z\in S$ that ``certifies" that. In fact, each $z$ can be thought of as a subset of indices, and if an index appears in $z$, then it cannot be in $I$. So we can build the perturbation in a way that every coordinate is elevated, unless $z$ certifies $i\not\in I$: In that case we define $h(w,z)$ so that its gradient will radically inhibit $i$.

The last observation is what becomes challenging in our case. As discussed, to achieve improved rate, we need to use Feldman's function. When using Feldman's function the coordinates cannot be ruled out, or identified, by a single $z$ but one has to look at the whole sample to identify $I$. While \citet{schliserman2024dimension} tackle a similar problem, 
we take a slightly different approach described next: For each $z$ assign a random, positive, number $\alpha(z)$. We can think of this $\alpha$ as a hash function. Let us add another coordinate to the vector $w$, $w(d+1)$. Consider the function
\[ h(w,z)= \gamma \alpha(z)\cdot w(d+1).\]
Then $\partial h(w,z) = \gamma \alpha(z)e_{d+1}$. Write $\alpha(S) = \frac{1}{|S|}\sum_{z\in S} \alpha(z)$ then in turn: \[
w_t(d+1)= w_{t-1}(d+1) -\partial \frac{\gamma}{|S|}\sum_{z\in S} h(w_{t-1},z) = - t\cdot\gamma\alpha(S)e_{d+1}.\]
If $\gamma,\alpha(z)$ are chosen correctly, $\alpha(S)$ is a one to one mapping from samples to real numbers, and small $\gamma$ ensures that the overall addition of $h$ has negligible affect on the outcome. Then, we may define the subgradient oracle to be dependent on coordinate $d+1$ which encodes the whole sample. Our final construction will take a different $h$, which adds small strong convexity in this coordinate, for reasons next explained:

\paragraph{Working with any first order oracle} Notice that our statement is slightly stronger than what we so far illustrated. \cref{thm:main} states that, for \emph{any} subgradient oracle, GD will fail. For that, we need to be a little bit more careful, and we want to replace our function with a function that leads to the same guaranteed trajectory, but at the same time it should be differentiable at visited points. This will ensure a unique derivative, making the construction independent of the choice of (sub)gradient oracle.

Towards this goal, we start with the construction depicted so far and consider the set of all values, gradients, and points $\{f_j, g_j, w_j\}_{j\in J}$ that our algorithm may visit, for any possible time step and any possible sample, with our construction. Notice that, while this set may be big and even exponential, it is nevertheless finite. What we want is to interpolate a new function through these triplets. In contrast with our original construction, we require a differentiable function at the designated points. Notice, that such an interpolation will have the exact same behaviour when implementing GD on it (with the added feature that the oracle is well defined and unique).

The problem of convex interpolation is well studied, for example \citet{taylor2017smooth} shows sufficient and necessary conditions for interpolation of a smooth function. Our case is slightly easier as we do not care about the smoothness parameter. On the other hand we do require Lipschitzness of the interpolation. We therefore provide an elementary, self-contained, proof to the following easy to prove Lemma, (proof is provided in \cref{prf:interpolation})
\begin{lemma}\label{lem:interpolation}
    Let $G=\{f_j,g_j,w_j\}_{j\in J} \subseteq \mathbb{R}\times\mathbb{R}^d\times\mathbb{R}^d$ be a triplet of values in $\mathbb{R}$, and gradients and points in $\mathbb{R}^d$, such that $\|g_j\|\le L$. Suppose that for every $i,j\in J$: 
    \begin{equation}\label{eq:ineq} f_i \ge f_j + g_j^\top (w_i-w_j),\end{equation}
    and let \[I_{\textrm{diff}}= \{i: f_i=f_j+ g_j^\top(w_i-w_j) \Rightarrow g_i=g_j\}.\]
    
    Then there exists a convex $L$-Lipschitz function $\hat {f}$ such that for all $j\in J$: $\hat f(w_j)=f_j$, and for all $i\in I_{\textrm{diff}}$, $\hat f$ is differentiable at $w_i$ and:
    \[ \nabla f(w_i)= g_i.\]
\end{lemma}
 
With \cref{lem:interpolation} at hand, consider the function
 
 \[h(w,z) = \frac{1}{2}(w(d+1))^2 +\alpha(z)\cdot w(d+1).\] 
 The above function encodes in $w(d+1)$ the sample and time-step as before. Moreover, because it is slightly strongly convex (in coordinate $d+1$), $w_1(d+1)\ne w_2(d+1)$ ensures that 
 \[  h(w_1,z) > h(w_2,z) + \nabla h(w_2,z)^\top(w_1-w_2),\] 
Then the term $h$ in $f$ ensures that the triples $\{f_j,g_j,w_j\}$ along the trajectory generate gradient vectors that satisfy strict inequality in \cref{eq:ineq} and in turn, our interpolation from \cref{lem:interpolation} will be differentiable at these points. There's some technical subtlety because the interpolation needs to also take the averaged iterate into account, but this is handled in a similar fashion.

In the next two sections we provide more formal statements of the two main ingredients: First, we define a setup of optimization with a sample-dependent first order Oracle and state a lower bound for the generalization error in this setup. The second ingredient is a reduction from the standard setup of first order optimization.

\subsection{Sample-dependent Oracle}\label{sec:sdo}
As discussed, the first step in our proof is to consider a slightly weaker setup where the first-order oracle may depend on the whole sample. Let us formally define what we mean by that. Define \[\S^T_m= \{\SS=(S_1,\ldots, S_t), S_i\in \cup_{i=1}^m \Z^{m}, t\le T\},\] the set of all subseqences of samples of size at most $m$.
Given a function $f(w,z)$, a sample dependent oracle, $\OSS$, is a finite sequence of first order oracles \[\OSS= \{\OS{t}(S; w,z)\}_{t=1}^T,\] that each receive as input a finite sample $S$, as well as $w$ and returns a subgradient: 

\[\OS{t}(S, w,z)\in \partial f(w,z).\] 
The sequence of samples can be thought of as the past samples that were observed by the algorithm. In the case of full-batch GD these will be the whole sample, and for SGD, each $S$ provided to $\OS{t}$ will be all previously observed samples.
Given $\SS\in \S_m^{T}$ Let us also denote
\begin{equation}\label{eq:ost} \OS{t}(\SS,w) = \frac{1}{|S_{t}|} \sum_{z\in S_t} \OS{t}(S_{1:t-1},w,z) \in \partial \left(\frac{1}{|S_t|} \sum_{z\in S_t} f(w,z)\right),\end{equation} 
where we let $\SS_{1:0}=\emptyset$, and $S_{1:t-1}= (S_1,\ldots, S_{t-1})$ is the concatenated subsample of all previously observed samples in the sequence. As discussed, working with a sample-dependent oracle is easier (for lower bounds). And indeed, our first result shows that, if the subgradient can be chosen in a way that depends on the sample, we can provide the desired lower bound. 
For fixed and known $\eta>0 $, $T$, a sample dependent first order oracle $\OSS$, and a sequence of samples $\SS=(S_1,S_2,\ldots, S_T)$, define $w_0=0$ and inductively: 
 \begin{align*} &w^\SS_{t} = \Pi\left[w^\SS_{t-1} - \eta\OS{t}(\SS,w^\SS_{t-1})\right], \\
  \textrm{and for every suffix  $\ss < T$:\quad}&\\
 &w_{\SS,\ss}^{GD} = \frac{1}{T-\ss}\sum_{t=\ss+1}^T w_{t}^\SS \labelthis{eq:s_dependent_gd}
\end{align*}

\begin{lemma}\label{lem:main}

    For every $m,d,T\ge 18$ and $\eta>0$ there are a distribution $D$, a $3$-Lipschitz convex function $f(w,z)$ in $\mathbb{R}^d$, as well as a sample dependent first order oracle $\OSS$ such that: if $\SS= (S,S,\ldots S)\in \S_m^T$ for $S\sim D^{m}$ i.i.d, then w.p. $1/2$, for every suffix averaging $\ss$:
    \[ F(w_{\SS,\ss}^{GD}) - F(0) \ge \theboundpre.\]

\end{lemma}
The proof of \cref{lem:main} is provided in \cref{prf:lem:main}. We next move to describe the second ingredient of our proof.
\subsection{Reduction to sample-dependent oracles}\label{sec:reduce}
As discussed, the second ingredient of our proof is a reduction to the sample-dependent setup. Instead of using a perturbation function as in \cite{amir2021sgd}, we take a more black box approach and show that, given a sample dependent first order oracle, there exists a function that basically induces the same trajectory. Proof is provided in \cref{prf:reduce}.
\begin{lemma}\label{lem:reduce} Suppose $q \in \mathbb{R}^T$, $\|q\|_\infty \le 1$. And suppose that $f(w,z)$ is a convex, $L$-Lipschitz, function over $w\in \mathbb{R}^d$, let $\eta>0$, let $\OSS$ be a sample dependent first order oracle, and for every sequence of samples $\SS=(S_1,S_2,\ldots, S_T)$ define the sequence $\{w_{t}^\SS\}_{t=1}^T$ as in \cref{eq:s_dependent_gd}.
    
    Then, for every $\epsilon>0$ there exists an 
    $L+1$ Lipschitz convex function\footnote{i.e. $w\in \mathbb{R}^d$ and $x\in \mathbb{R}$} $\bar f((w,x),z)$ over $\mathbb{R}^{d+1}$ (that depends on $q,f,T,\eta,m,\OSS,\epsilon$).
    
    such that
    for any first order oracle $\O_z$ for $\bar f$, define $u_0=0 \in \mathbb{R}^d$ and $x_0=0 \in \mathbb{R}$, and:
    \[ (u_t, x_t)  = (u_{t-1},x_{t-1}) - \frac{\eta}{|S_t|} \sum_{z\in S_t} \O_z((u_t,x_t))\] then if we define:  
    \[ u_{\q}= \sum_{t=1}^T \q(t)u_t, ~\textrm{and},\quad x_{\q}= \sum_{t=1}^T \q(t)x_t,~\textrm{and} ~w_{\q}^\SS= \sum_{t=1}^T \q(t) w_t^\SS.\] 
    then, we have that $u_{\q}=w_{\q}^\SS$ and:
    \[ |\bar f((u_{\q},x_{\q}),z)- f(w_{\q}^\SS,z)| \le \epsilon.\]
    and,
    \[ |\bar f((0,0),z)- f(0,z)| \le \epsilon.\]    
\end{lemma}

\section{Proof of \cref{thm:main}}
The proof is an immediate corollary of \cref{lem:main,lem:reduce}, which we prove in \cref{prf:lem:main,prf:reduce} respectively. To see that \cref{thm:main} indeed follows from these Lemmas, start with $\eta, d,m,T$ that satisfy the conditions. Let $f(w,z)$ be the function and $\OSS$ the sample dependent first order oracle, whose existence follows from \cref{lem:main} with suffix $\ss=0$. And let $\bar f$ be the function whose existence follows from \cref{lem:reduce} to some arbitrarily small $\epsilon_0$, with $\q(t)=\frac{1}{T}$ for all $t$. It is easy to see that if we apply GD on $\bar f$ and define its output $(u^{GD},x^{GD})$ then: $(u^{GD},x^{GD}) = (u_{\q}, x_{\q})$, and $w^\SS_{\q} = w_\SS^{GD}$.

Then, we have that w.p. $1/2$:

\begin{align*} \bar F((u^{GD},x^{GD})) - \bar F(0) & = \bar F(u_q,x_q)- \bar F(0) \\ &\ge  F(w_\SS^{GD}) -2\epsilon_0 - F(0)  
\\
&\ge 
\theboundpre - 2\epsilon_0
 \\
 &\ge 
 \thebound
.\end{align*}
Where in the last equation, we assume $\epsilon_0$ to be sufficiently small.

Notice, that by the same argument, by taking any suffix $\ss<T$, and setting $q(t)=0$ for $t\le \ss$, and $q(t)=\frac{1}{T-\ss}$ for $t\ge \ss+1$, we can obtain the following stronger result for any suffix averaging:

\begin{theorem}\label{thm:main_suffix}
For every $d\ge 4096, T\ge 10, m\ge 1$ and $\eta>0$ and suffix $\ss <T$, there exists a distribution $D$, and a $4$-Lipschitz convex function $f(w,z)$ in $\mathbb{R}^{d}$, such that for \emph{any} first order oracle of $f(w,z)$, with probability $1/2$, if we run GD with $\eta$ as a learning rate then:
 \[ F\left(\frac{1}{T-\ss}\sum_{t=\ss+1}^T w_t^\SS \right) - F(0) \ge \thebound.\]

\end{theorem}

\subsection{Proof of \cref{lem:main}}\label{prf:lem:main}
For simplicity we will assume that $d=2^{n}$ for some $n\in \mathbb{N}$, the final result will be obtained by embedding a construction in a subspace of size at least half the original dimension.

We start by recalling Feldman's construction \cite{feldman2016generalization}: There exists a set $\V \subseteq \{0,1\}^d$, such that:
\begin{equation}\label{eq:V} \left(\forall v_1\ne v_2\in \V,~ v_1\cdot v_2 \le \frac{5d}{16}\right) \textrm{,~and}~\left( \|v\|^2\ge \frac{7d}{16}\right)\textrm{,~and~}\left( |\V|> e^{d/258}\right).\end{equation}

Indeed, suppose we pick randomly $w\in \{0,1\}^d$ according to probability $P$ where each coordinate $P(w(i))=1$ with probability $1/2$ (independently). Then, by Hoeffding's inequality for two $w_1,w_2\sim P$ independently:
 \[ P\left (w_1\cdot w_2 > \frac{d}{4} + \frac{d}{16} \right) \le e^{-\frac{d}{128}},\]
 \[ P\left(w_1 \cdot w_1 < \frac{d}{2} - \frac{d}{16}\right) \le e^{-\frac{d}{128}}.\]
 Thus, picking $\V$ elements i.i.d according to $P$, randomly, of size $|\V|\ge  e^{d/258}$ we can show by union bound that with positive probability, all $|\V|^2$ pairs in $\V$ will satisfy the requirement.
 Next,  we define a distribution $D_\epsilon$ supported on subsets of $\V$ such that for a random variable $V\subseteq \V$ each $v\in V$ w.p. $\epsilon$ independently. 
We start by assuming that $T\le \frac{1}{17} d^3$ (the case $T>d^3$ is handled at the end), and let $k\in \mathbb{N}$ be such that:

\[ d \le k\left(\frac{T}{17}\right)^{1/3} < 2d.\]
One can show that without loss of generality we can assume $k$ is also a power of $2$ (in particular, $d$ is divisble by $k$ for large enough $T$).
We next follow the idea depicted in \cref{sec:overview}, but we want to handle the case $k\gg 1$. For that, we redefine the function in \cref{eq:nemdef1}, and take blocks of coordinates. To simplify notations, let us define for two set of indices $I,J$ of elements in $[d]$: $I=\{i_1,\ldots,i_k\}$, $J= \{j_1,\ldots, i_k\}$, $I\prec J$ if $\max\{i\in I\} < \min\{j\in J\}$. and we will also write:
\[ e_{I} = \frac{1}{\sqrt{|I|}}\sum_{i\in I} e_i, \quad \textrm{and}, w(I) = w\cdot e_I = \frac{1}{\sqrt{I}}\sum_{i\in I} w(i).\]
then we define our final function as:
\begin{equation}\label{eq:nemonsterdef}
N(w)=\max\left\{0,\max_{|I|= k}\{-w(I)\},\max \left\{- (w(I)- w(J)): I\prec J, |I|=|J|=k\right\}\right\}
\end{equation}

Define $\alpha=\min\{\frac{1}{\eta \sqrt{2T}},1\}$, and let:

\[f(w,V)= g(w,V) + \alpha N(w).\]
where $N$ is defined in \cref{eq:nemonsterdef}, and $g$ is defined to be Feldman's function with a suitable choice of threshold: 
\begin{equation}\label{eq:gs} g(w,V) = \frac{1}{\sqrt{d}} \max_{v\in V}\left\{ \frac{45 \eta\alpha  d^2}{2\cdot 16^2k^{1.5}}  ,  w\cdot v\right\}.\end{equation}
Notice that $N$ is $2$-Lipschitz, and $g$ is $1$-Lipschitz.

To obtain the trajectory, we next define a sample dependent oracle. We only define it for samples $\SS$ such that there exists $v^\star \notin V_i$ for all $V_i\in S$ (define it arbitrarily to any other type of sample). Let $\mathcal I=\{i_1,\ldots, i_{d'}\}$ be a set of $\frac{7d}{16}$ 
indices such that $v^\star(i_j)\ne 0$. Divide the elements of $\mathcal I$ into $d'/k$ subsets. Namely, let 
\[ I_j = \{i_{(j-1)\cdot k+1},i_{(j-1)\cdot k+2}, \ldots,i_{j\cdot k}\},\quad j=1,\ldots, d'/k.\]

We start by defining only an oracle for the function $\alpha N$. We will later show that the trajectory induced by this oracle stays in the minima of $g$, and that will show that, for our purposes, we can choose the same oracle for the whole function $f$.
We denote by $\OS{t}_{\alpha N}$ the sample dependent oracle for $\alpha N$, and we first define: 
\[ \OS{1}_{\alpha N}(\emptyset,0)=0,\]

 Next, we define for $t>1$. For any $w$ such that $0\notin \partial N(w)$, define it arbitrarily. If $0\in \partial N(w)$ we define it as follows: 

\begin{itemize}
\item If there is a multi-index $I_j$ that $w(I_j)=w(I_{j+1})> \eta\alpha$: Then:
\[ \OS{t}_{\alpha N}(S,w)= \alpha(e_{I_{j+1}} - e_{I_j}).\]
\item If there is no such multi-index, and if $I_j$ is the minimal multiindex such that $w(I_j)=0$, set:
\[ \OS{t}_{\alpha N}(S,w) = - \alpha(e_{I_j}).\]
\item If both conditions cannot be met, then:
\[ \OS{t}_{\alpha N}(S,w) = 0.\]
\end{itemize}
The trajectory of this dynamic is depicted in \cref{fig:nemonster}, for the case $k=1$. In the general case, we can think of each coordinate in \cref{fig:nemonster} as a block of size $k$. Now we assume that $\{w_t^\SS\}_{t=1}^T$ follows the trajectory depicted in \cref{eq:s_dependent_gd} with that choice of Oracle. It can be seen that the update step is such that after $T'=1+ \sum_{t'=1}^{d'/k} \sum_{t''=1}^{t'} t''$ rounds, we will have that for every $i\in I_t$:

\begin{equation}\label{eq:nemonster} w_{{T'}}^\SS(i) = \alpha\eta\sqrt{k} (d'/k+1-t),\end{equation}
and that for $t>T'$: $w^\SS_{t} = w^\SS_{T'}$

Moreover, one can show that $w_t$ is non zero only in coordinates $i\in \mathcal I$, and that for any subset $I_B\subseteq I$ such that $|I|=B$:
\begin{equation}\label{eq:decreasing} \sum_{i\in I_B} w_{t}^\SS(i) \le \max\left\{ \sum_{i\in I_B} w_T^\SS(i): I_B\subseteq I, |I|=B\right\}\end{equation}

 A formal proof is provided in \cref{prf:nemonster}.
Notice that:

\[ 1+ \sum_{t'=1}^{d'/k} \sum_{t''=1}^{t'} t'' < \frac{{d'}^3}{k^3} \le \frac{1}{5}T.\]

Next, for any $v\ne v^\star$, we have that:

\begin{align*} w_t^\SS \cdot v &\le  \sum_{i=1}^{d'/k}  w_t^\SS(i_t) \mathbf{1}\{v(i_t)=1\}\\
&\le  \max \left\{\sum_{i\in I_B} w_{T}^\SS(i): I_B\subseteq I, |I|\le \frac{5d}{16}\right\} &\cref{eq:decreasing,eq:V}\\
&\le \sum_{t=1}^{\frac{5d}{16k}} \sum_{i\in I_t} w_T^\SS(i)\\
&\le 
\sum_{t=1}^{\frac{5d}{16k}} \sqrt{k}\eta\alpha(\left(d'/k\right)+1-t) &\cref{eq:nemonster}\\
&\le
\sum_{t=0}^{\frac{5d}{16k}} \sqrt{k}\eta\alpha(\left(d'/k\right)-t) \\
&\le \sqrt{k}\eta\alpha \left( \frac{5d}{16k}\frac{d'}{k}- \frac{1}{2}\left(\frac{5d}{16k}\right)^2\right)\\
\\
&\le \sqrt{k}\eta\alpha  \frac{45d^2}{2(16k)^2}& d' =\frac{7d}{16}\\
.\end{align*}
As such $0 \in \partial g(w_t^\SS,V_i)$ for all $V_i$, and we can define $\OS{t}$ so that
\[\OS{t}(S, w_t^\SS,V_i) = \OS{t}_{\alpha N}(S,w_t^\SS).\]

And we have that, $\{w_t^\SS\}$ is the trajectory obtained from \cref{eq:s_dependent_gd} with respect to this oracle of $g$ also. We also have:

\[ w_{T'}^\SS\cdot v^\star =  \sqrt{k}\eta\alpha  \sum_{t=1}^{d'/k}  (d'/k+1-t) = \sqrt{k}\eta\alpha  \sum_{t=1}^{d'/k}  t  \ge \frac{\sqrt{k} \eta\alpha}{2}\frac{{d'}^2}{k^2}  \ge \sqrt{k}\eta\alpha \frac{49 d^2}{2(16k)^2} .\]
And because $T'\le \frac{1}{17}T$, for every suffix $\ss\in [T]$:
\[ v^\star \cdot w_{\SS,\ss}^{GD} \ge \frac{16}{17} w_{T'}^\SS\cdot v^\star \ge \sqrt{k}\eta\alpha \frac{46 d^2}{2(16k)^2}.\] 
Recall that we assume $\frac{136d^3}{k^3} \ge T$: 
\begin{equation}\label{eq:genbound} F(W_{\SS,\ss}^{GD}) - F_S(0) \ge
\epsilon \left(\frac{46\sqrt{k}\alpha\eta d^{1.5}}{2(16k)^2} -\frac{45\sqrt{k}\alpha\eta d^{1.5}}{2(16k)^2}\right)\ge \epsilon \frac{\alpha\eta }{2\cdot 16^2}\left(\frac{d}{k}\right)^{1.5}
\ge \frac{\epsilon}{\sqrt{2} \cdot 272 \cdot 16^2}\min\left\{\eta\sqrt{T},1 \right\}\end{equation}

\cref{eq:genbound} lower bounds the generalization under the event that there exists $v\in \V$ such that $v\notin V_i$ for every $i=1,\ldots, m$. Now assume $\epsilon = \min\{\frac{d}{516m},\frac{1}{4}\}$, then for every $v$, using the inequality $(1-\epsilon)\le e^{-2\epsilon}$ for $\epsilon<1/2$:
\[ P(v\notin \cap_{i=1}^m  V_i)= (1-\epsilon)^m\ge e^{-2\epsilon\cdot m} \ge e^{-2d/516},\]
and,
\[ P(\exists v, \notin \cap_{i=1}^m  V_i) \ge 1 - (1 - e^{-2d/516})^{|V|} \ge 1 - (1 - e^{-d/258})^{e^{d/258}}\ge 1- e^{-1} \ge 1/2.\]

So far we assume that $d=2^n$, notice that if $T< \frac{1}{8\cdot 17} d^3$, we can find a subspace of size $d_1 > \frac{d}{2}$ so that $T< \frac{1}{8\cdot 17}d_1^3$, and we obtain our final result by embedding our construction in this subspace.

\paragraph{When $T> \frac{1}{17\cdot 8} d^3$}
Notice that if we take the construction with $T= \floor{\frac{1}{17\cdot 8} d^3}$ then:
\[\OS{T}(w^\SS_T)=0,\]
hence we can use the above construction for any $T'>T$, and the subdifferential oracle is defined for every iteration above $T$ as returning $0$, and we obtain a similar analysis to the case that $T=\floor{d^3/136}$, and our bounds yield as in \cref{eq:genbound}:
\begin{equation*} F(w_{\SS,\ss}^{GD}) - F_S(0) \ge \frac{\epsilon}{\sqrt{2}\cdot 272\cdot 16^2} \min\left\{\eta\sqrt{\floor{d^3/136}},1\right\}\end{equation*}

\subsection{Proof of \cref{lem:reduce}}\label{prf:reduce}
The main ingredient in our proof is the following claim whose proof is provided in \cref{prf:xtrajectory}. To state the claim we will need a few notations. First, for any two tuples of sample-sequences in $\S^T_m$, and $t>0$ $(\SS,t)$ and $(\SS',t)$, let us denote $(\SS,t)\equiv (\SS',t')$ if $t=t'$ and $S_{1:t} = S'_{1:t}$, namely the prefixes of the sample sequences agree up to point $t'=t$. Second, for a mapping $\alpha: \Z\to [0,1]$, from sample points to real numbers, and a sample $S$, let us denote
        \[ \alpha(S)= \frac{1}{|S|}\sum_{z\in S} \alpha(z).\]
    \begin{claim}\label{clm:xtrajectory}
For every $\eta,\gamma$ and  $\alpha$ define for every sequence of samples $\SS= (S_1,\ldots, S_T)$, and $t$ inductively: $x^\SS_0=0$ and 
        \begin{equation}\label{eq:xts} x^\SS_t = (1-\gamma \eta) x^\SS_{t-1} - \gamma \eta \alpha(S_t).\end{equation}
        and define:
        \begin{equation}\label{xss} x_{\q}^{\SS}= \sum_{t=1} \\q(t) x^{\SS}_{t} \end{equation}
        Then, 
        for any $\epsilon>0$, there is a choice of $\gamma<\epsilon $ and $\alpha$ such that:
        \begin{enumerate}
            \item\label{it:xts1} For $(\SS,t)\not\equiv (\SS',t')$: $ x^\SS_t \ne x^{\SS'}_{t'}.$
            \item\label{it:xts2} For every $\SS,t$, such that $t<\max\{t': q(t')\ne 0\}$ and $\SS'$: $x_{t}^\SS \ne \sum_{t=1}^T \q(t) x_t^{\SS'}$.
        \end{enumerate}
    \end{claim}
    The $x_t^\SS$ represents the different states the trajectory can be in. $x_q^\SS$ represent the output of the trajectory which can be an aggreagated sum. We require that each provide a signature for the state, and this will allow us to ``code" the state of the trajectory along GD.

    We now continue with the proof of \cref{lem:main}. For every sample dependent first order oracle $\OS{t}$, we have that:
    \begin{equation}\label{eq:subgradsum} \OS{t}(\SS;w) = \frac{1}{|S_t|}\sum_{z\in S_t} \OS{t}(S_{1:t-1};w,z).\end{equation}
   
To simplify notations, let us denote:
    \[\OS{t}(S_{1:t-1},w_t^\SS,z) = \O_{\SS,t,z},\]
as $\SS,t,z$ completely determine the output. We next define \[h_z(x) := \frac{1}{2}\gamma (x^2-  2\alpha(z)x) ,\]
where $\gamma>0$ is arbitrarily small\footnote{$\gamma \le \epsilon/(\eta T)$, will suffice}.

\ignore{Observe that
\begin{equation}\label{eq:hdynamic} x_{t}^\SS = x_{t-1}^\SS - \frac{\eta }{|S_{t}|}\sum_{z\in S_{t}}\nabla h_z(x_{t-1}^\SS) = (1-\gamma \eta)x_{t-1}^\SS -\gamma\eta \alpha(S_t).\end{equation} 
In particular, by \cref{clm:xtrajectory}, if $x_0^\SS=0$:
\begin{equation}\label{eq:clmreiterated} (\SS,t)\not\equiv (\SS',t') \Leftrightarrow ~ x_{t}^\SS \ne x_{t}^{\SS'} ,\end{equation}}
and observe that, if $x_t^\SS$ is defined as in \cref{eq:xts}:
\begin{equation}\label{eq:hdynamic} x_{t}^\SS = (1-\gamma \eta)x_{t-1}^\SS-\gamma\eta \alpha(S_t) = x_{t-1}^\SS- \frac{\eta}{|S_t|} \sum_{z\in S_t}\nabla h_z(x_{t-1}^\SS).\end{equation}
To simplify notation, let us denote:
\[ w_{T+1}^\SS = w_q^\SS, ~\textrm{and}~ x_{T+1}^\SS= x_q^\SS,\]
and assume without loss of generality that $\max\{q(i) \ne 0\}= T$ (Otherwise, we look only at the sequence up to point $\max\{q(i) \ne 0\}$). Consider now the sets of triplets:
    \[ G(z) = \left\{ (v,g,u)= \left( f(w_t^\SS,z)+h_z(x_t^\SS), (\O_{\SS,t,z},\nabla h_z(x_t^\SS)), (w_t^\SS,x_t^\SS)\right) :  \SS_T^m\in \S, t\le T+1\right\},\]
where  $\O_{\SS,T+1,z} \in \partial (f(w_{\q}^\SS)+h_z(x_{\q}^\SS))$ is chosen arbitrarily.

Convexity of $f+h_z$ ensure that the triplets in $G(z)$ satisfy \cref{eq:ineq} for all $t\le T+1$, as in \cref{lem:interpolation}. To apply the Lemma, we also want to achieve differentiability at points such that $t< T$. Therefore, take any two triplets \[(v_i,g_i,u_i) =\left( f(w_{t_i}^\SS,z)+h_z(x_{t_i}^\SS), \left(\O_{\SS,t,z},\nabla h_z(x_{t_i}^\SS)\right),\left(w_{t_i}^\SS,x_{t_i}^\SS\right)\right), i={1,2},\]
where $t_1 <T$ and $t_2\le T+1$, and suppose $g_1\ne g_2$. To simplify notations, let us write $w_{t_i}^\SS = \wz{i}$ and $x_{t_i}^\SS=\xz{i}$.

First, by convexity of $f$ we have that:
\begin{align*}
v_1-v_2+g_2^\top(u_2-u_1)=
&\f{1}+\h{1} - \f{2}-\h{2} - \g{2}^\top\left(\u{1}-\u{2}\right)\\
= &\f{1}-\f{2}-\oz{2}^\top\left(\wz{1}-\wz{2}\right)+\h{1}-\h{2} - \hz{2}^\top\left(\xz{1}-\xz{2}\right)\\
\ge&
\h{1}-\h{2} -\hz{2}^\top\left(\xz{1}-\xz{2}\right).
\end{align*}

Next, because $g_1\ne g_2$, either 
$\hz{1}\ne \hz{2}$, which implies 
$x_1
\ne x_2$, or $\O_{\SS_{1},t_1,z}\ne \O_{\SS_{2},t_2,z}$ which implies $(\SS_1,t_1)\not\equiv (\SS_2,t_2)$ which again implies $x_1 \ne x_2$ by \cref{clm:xtrajectory}. In other words, if $g_1\ne g_2$ then $\xz{1}\ne \xz{2}$:

\begin{align*}
\h{1}-\h{2} -\hz{2}^\top\left(\xz{1}-\xz{2}\right)&= \gamma\left(\xz{1}^2-2\alpha(z)\xz{1}-\xz{2}^2-2\alpha(z)\xz{2}-(2\xz{1}-2\alpha(z))\cdot(\xz{1}-\xz{2})\right)  \\
& = \gamma\left(\xz{1}^2 + \xz{2}^2 - 2\xz{1}\cdot \xz{2} \right)\\
& = \gamma\left(\xz{1}-\xz{2}\right)^2\\
& > 0 &\xz{1}\ne \xz{2}
\end{align*}

We showed then, that $v_1-v_2+g_2^\top (u_2-u_1)=0$, implies $g_1=g_2$.
We obtain, by \cref{lem:interpolation}, that there exists a function $\bar f((w,x),z)$ such that for all $t$ and $\SS$:
\begin{equation}\label{eq:prv1} \bar f(w_t^\SS,x_t^\SS,z)= f(w_t^\SS,z)+h_z(x_t^\SS)
\end{equation}
and for all $t\le T$,
\begin{equation}\label{eq:prv2}
\nabla  \bar f((w_t^\SS,x_t^\SS)) = (\O_{\SS,t,z},\nabla h_z(x_{t}^\SS))
.\end{equation}

This proves \cref{lem:reduce}.
Indeed. By the Lipschitzness of $h$ in the unit ball, we have that $|x_t|\le \gamma \eta T$. For sufficiently small $\gamma$, from \cref{eq:prv1}, since $\{(f(0,z)+h_z(0),(0,\nabla h_z(0)),(0,0))\}\in G(z)$, we obtain that 
\[|\bar f((0,0),z)- f(0,z)|= \gamma |h_z(0)|\le \gamma^2 \eta T\le \epsilon.\]
\[|\bar f((w_{\q}^\SS,x_{\q}^\SS),z)- f((w_{\q}^\SS,z))|= \gamma |h_z(x_{\q}^\SS)|\le \epsilon.\]
Further, if we assume by induction that for every $t'\le t-1$, $u_{t'}=w^{\SS}_{t'}$, and $x_{t'}=x_{t'}^\SS$, then from \cref{eq:prv2} we have that for any first order oracle:
\begin{align*} (u_t,x_t) &= (w_{t-1}^\SS,x_{t-1}^\SS)-\frac{\eta}{|S_t|}\sum_{z\in S_t} \O_z(w^\SS_{t-1},x_{t-1}^\SS)
\\
&= (w_{t-1}^\SS,x_{t-1}^\SS)-\frac{\eta}{|S_t|}\sum_{z\in S_t}\left(\O_{\SS,t,z},\nabla h_z(x_{t-1}^{\SS},z)\right)
\\
&= (w_{t-1}^\SS,x_{t-1}^\SS)-\frac{\eta}{|S_t|}\sum_{z\in S_t}\left(\OS{t}(S_{t-1},w^\SS_{t-1},z),\nabla h_z(x_{t-1}^{\SS},z)\right)
\\
& = (w_{t-1}^\SS,x_{t-1}^\SS)-\eta \left(\O(\SS,w^\SS_{t-1}),\frac{1}{|S_t|} \sum_{t=1}^T \nabla h_z(x_{t-1}^{\SS},z)\right)
\\
\\
&= \left(w_{t-1}^\SS - \eta O(\SS,w_{t-1}^\SS), x_{t-1}^\SS - \frac{\eta}{|S_t|}\sum \nabla h_z(x_{t-1}^\SS)\right) \\
& = (w_t^{\SS},x_{t}^\SS), &\cref{eq:hdynamic}
\end{align*}
which proves, by linearity, that $u_{\q} = w_{\q}^\SS$.

\paragraph{Ackgnoweledgments}
The author would like to thank Tamar Livni for creating Figure \ref{fig:nemonster}. Tamar holds all copyrights to the artwork. The author would also like to thank Tomer Koren and Yair Carmon for several discussions. This research was funded in part by an ERC grant (FOG, 101116258), as well as an ISF Grant (2188 $\backslash$ 20).
 \bibliographystyle{abbrvnat}
 \bibliography{bibbop}
 
\appendix
\section{Proof of \cref{lem:interpolation}}\label{prf:interpolation}
We choose
\[ \hat f(w) = \max_{j\in J}\{f_j+ g_j^\top(w-w_j\}.\]
$\hat f$ is indeed convex as it is the maximum of linear functions. Further, it is known \cite{rockafellar2015convex} that at each point $w$:
\begin{equation}\label{eq:conpartial} \partial \hat f(w) = \mathrm{conv}\{g_j: \hat f(w)= f_j+ g_j^\top(w-w_j)\}.\end{equation} It follows that, $\hat f$ is $L$-Lipschitz. Next, for any $w_i$, notice that our assumption implies:
\[ f_i \ge \max_{j\in J}\{f_j+ g_j^\top(w_i-w_j\}= \hat f(w_i).\] On the other hand,
\[f_i = f_i + g_i^\top(w_i-w_i)\le \hat f(w_i).\]
Hence $\hat f(w_i)= f_i$. Finally, to see the function is differentiable at designated points, take any $i\in I_{diff}$ and consider $w_i$. By \cref{eq:conpartial}, it is enough to show that if $g_j \in \partial f(w_i)$ then $g_i = g_j$, but this clearly follows from our assumption, and the fact that $\hat f(w_i)=f_i$.

\section{Proof of \cref{clm:xtrajectory}}\label{prf:xtrajectory}

We first prove by induction that $x_0^\SS=0$ and, for $t\ge 1$:
\begin{equation}\label{eq:xst}
    x_t^{\SS} =  \gamma \eta\sum_{z\in Z}\alpha(z)\left(\sum_{\{t'\le t,z\in S_{t'}\}} \frac{(1-\gamma\eta)^{t-t'}}{|S_{t'}|}\right) .
\end{equation}
Indeed,
\begin{align*}
    x_t^{\SS}&= x_{t-1}^{\SS} -\gamma\eta \left(\eta x_{t-1}^{\SS}- \alpha(\SS_t)\right) \\
    &= (1-\gamma \eta) x_{t-1}^\SS + \gamma \eta \alpha(\SS_t)
    \\
    &= (1-\gamma \eta)\left(\sum_{z \in Z} \alpha(z) \sum_{\{t'<t : z\in S_{t'}\}}\frac{(1-\gamma\eta)^{t-1-t'}}{|S_{t'}|} \eta \gamma \right)+ \frac{1}{|S_t|}\sum_{z\in S_t}\gamma \eta \alpha(z)\\
    &= \sum_{z \in Z} \alpha(z) \sum_{\{t'< t : z\in S_{t'}\}}\frac{(1-\gamma\eta)^{t-t'}}{|S_{t'}|} \eta \gamma + \frac{\mathbf{1}[ z\in \SS_t]}{|S_t|}\gamma \eta\\
       &=  \eta \gamma  \sum_{z \in Z}\alpha(z)  \sum_{\{t'\le t : z\in S_{t'}\}}\frac{(1-\gamma\eta)^{t-t'}}{|S_{t'}|}
\end{align*}

Now, for every $\SS, t, z$, define a polynomial:

\[ P_{\SS,t,z}(X) = \sum_{n=0}^{t-1} \frac{\mathbf{1}[z\in S_{t-n}]}{|S_{t-n}|}X^{n},\]
and let $r$ be a rational point, sufficinetly small, so that $1-r$ is \emph{not} the root of any polynomial of the form $P_{\SS,t,z}-P_{\SS',t',z'}$ that is distinct from $0$. In other words, we choose $r$ so that

\[ P_{S,t,z}(1-r) = P_{S',t',z'}(1-r) \Leftrightarrow  P_{S,t,z}(X) = P_{S',t',z'}(X).\]
and we also require that
\[ \sum_{t=1}^T q(t) P_{S,t,z}(1-r) = P_{S',t',z}(1-r) \Leftrightarrow  \sum_{t=1}^T q(t) P_{S,t,z}(X) = P_{S',t',z'}(X).\]
Notice that there are only finitely many polynomials of the above form, hence we can choose such $r$ in any interval $(0,\epsilon)$ for any $\epsilon>0$.
Rewriting \cref{eq:xst} we have:
\[ x_t^\SS = \gamma\eta \sum_{z\in Z} \alpha(z)P_{\SS,t,z}\left(1-\gamma\eta\right)\]

Now, suppose we choose $\{\alpha(z)\}$ to be reals, independent over the rationals, and suppose we choose $\gamma = r/\eta$. 

\paragraph{Proof of \cref{it:xts1}}
Assume that $x_t^\SS = x_{t'}^{\SS'}$. Because $\alpha(z)$ are independent over the rationals, and because $P(1-\gamma\eta)$ are always rationals, we have that $P_{\SS',t',z}(1-r) = P_{\SS',t',z}(1-r)$ for every $z$. But then, \[ \forall z\in Z: P_{\SS,t,z}(X)= P_{\SS',t',z}(X),\] by choice of $r$.
But then, $t=t'$. Indeed, assume by contradiction and w.l.og assume $t< t'$: then for any $z\in S'_1$ $P_{\SS,t,z}$ is a $t'-1$-degree polynomial, on the other hand $P_{\SS,t,z}$ is at most of degree $t-1< t'-1$. 
Next, if $S_i\ne S'_i$, for $i\le t$, then we can assume (w.l.o.g) that there is $z\in S_i$ such that $z\notin S'_i$ it follows that, by looking at the coefficient of the two polymials of the monomial $X^{t-i}$ we have that:
\[ P_{\SS',t',z}(X) \ne  P_{\SS,t,z}(X).\]

Overall then, we obtain that if $x_{t}^\SS = x_{t'}^{\SS'}$ then $(\SS,t)\equiv (\SS',t').$.

\paragraph{Proof of \cref{it:xts2}:} The proof of \cref{it:xts2} is similar to how we proved $t=t'$. Notice that:
\[ \sum_{t=1} \q(t) x_t^{\SS} = \gamma\eta \sum_{z\in Z} \alpha(z) \sum_{t=1}^T \q(t) P_{\SS,t,z}(1-\eta\gamma).\] 
Hence $x_t^{\SS}= x^{\SS'}$ implies for all $z\in Z$:
\[ \sum_{t=1}^T \q(t) P_{\SS',t,z} = P_{\SS,t,z}.\] 
For $z\in S'_1$ we have that $\sum_{t=1}^T \q(t) P_{\SS,t,z}$ is a $\max\{t: \q(t)\ne 0\}$-degree polynomial, but on the other hand we assume that $t<\max\{t: \q(t)\ne 0\}$, hence $P_{\SS,t,z}$ is a lower degree polynomial.

\section{Proof of \cref{eq:nemonster,eq:decreasing}}\label{prf:nemonster}
\subsection{Proof of \cref{eq:nemonster}}
To avoid cumbersome notations, we will supress dependence on $\SS$ and write $w_t$ instead of $w_t^\SS$.

As a first step, observe that at each iteration no projection is performed. Indeed, let us show by induction that:
\[ \|w_{t+1}\|^2=\|w_t-\eta \OS{t}(S,w,V)\|^2 \le 2\eta^2\alpha^2 (t+1).\]
The definition of $\alpha$ then implies that $w_t$ are restricted to the unit ball without projections.

To see the above is true, let us consider the case where the first type of update is performed:
\begin{align*}\|w_{t+1}\|^2&= \|w_{t}^\SS-\eta \OS{t}(S,w,V)\|^2\\
&=\|w_{t}^\SS-\eta\alpha e_{I_{j+1}} +\eta\alpha e_{I_j}\|^2\\
&= \sum_{s=1}^{j-1} (w_t(I_s))^2 +
(w_t(I_{j})+\eta\alpha)^2 +(w_t(I_{j+1})-\eta\alpha)^2
+\sum_{s=j+2}^{\floor{d'/k}} (w_t(I_s))^2\\
&=\sum_{s=1}^{\floor{d'/k}} (w_t(I_s))^2+2\eta\alpha(w_t(I_j)-w_t(I_{j+1}))+2\eta^2\alpha^2\\
&= \sum_{s=1}^{\floor{d'/k}} (w_t(I_s))^2+2\eta^2\alpha^2 & w_t(I_{j+1})=w_t(I_{j})\\
&\le 2\eta^2\alpha^2 \cdot t+ \eta^2\alpha^2\\
&=2\eta^2\alpha^2 \cdot (t+1).
\end{align*} 
And if the second type of update is performed:
\begin{align*}\|w_{t+1}\|^2&= \|w_{t}^\SS-\eta \OS{t}(S,w,V)\|^2\\
&=\|w_{t}^\SS+\eta\alpha e_{I_j}\|^2\\
&= \sum_{s\ne j} w_t(I_s) + \eta^2\alpha^2 & w_t(I_j)=0\\
& \le 2\eta^2\alpha^2\cdot t +\eta^2\alpha^2\\
& \le 2\eta^2\alpha^2\cdot (t+1)
\end{align*}
We now move on to prove that \cref{eq:nemonster} holds by induction. Specifically, we will show that for $d_0 \le \floor{d'/k}$ that at time $T_{d_0} = 1+\sum_{d=0}^{d_0} \sum_{k=0}^{d} k$, for any $t\le d_0$:

\begin{equation}\label{eq:induction} w_{T_{d_0}}(I_t) =\alpha\eta \begin{cases}
    (d_0 + 1 -t) & t\le  d_0\\
    0 & \textrm{o.w.}
\end{cases} .\end{equation}
\cref{eq:nemonster} then follows by plugging $d_0 = d'/k$, and noting that at every step we have that if $i\in I_t$ then $w_{T_{d_0}}(i)=\sqrt{k} w_{T_{d_0}}(I_t)$. Therefore, we are left with proving \cref{eq:induction}.

For $d_0=0$, $T_0=1$ and we have, indeed, $w_1=0$. Next assume we proved the statement for $d_0$, and we will prove it for $d_0+1$. Here too, we will use induction, and we prove that, for $d_1\le d_0+1$, at time \[T_{d_0,d_1}= T_{d_0} + \sum_{k=0}^{d_1-1} (d_0+1-k),\] we have that:

\[ w_{T_{d_0,d_1}}(I_t) = \alpha\eta \begin{cases}
    (d_0+2-t) & t\le d_1\\
    (d_0+1-t) & d_1< t\le  d_0\\
    0& \textrm{o.w.}
\end{cases}\]
For the case $d_1$=0, $T_{d_0,d_1}=d_0$, and it follows from our outer-induction step. Assume the statement is true for $d_1$ and we will prove it for $d_1+1$, here, yet again, we use induction. And we will show that for $1 \le d_2< d_0+2-d_1$ we have at time \[T_{d_0,d_1,d_2}=T_{d_0} + T_{d_1} + d_2,\] we have that:

\begin{equation}\label{eq:sksd} w_{T_{d_0,d_1,d_2}}(I_t) = \alpha\eta \begin{cases}
    (d_0+2-t) & t \le  d_1\\
    (d_0 +1 -t) & d_1 < t < d_0 + 2 -  d_2\\
    (d_0+2 -t) & t= d_0 + 2 - d_2 \\
    (d_0+1-t) & d_0+2-d_2 < t\le  d_0\\
    0 & \textrm{o.w.}
\end{cases}\end{equation}
We start the induction with the case $d_2=1$, in that case notice that $T_{d_0,d_1,d_2}= T_{d_0,d_1}+1$, and by induction hypothesis:
\[ w_{T_{d_0,d_1}}(I_t) = \alpha\eta \begin{cases}
    (d_0+2-t) & t\le d_1\\
    (d_0+1-t) & d_1< t\le  d_0\\
    0 & \textrm{o.w.}
\end{cases}\]
In this case, note that there are no two consecutive coordinates that are equal, hence our choice of oracle is defined so that $\OS{t} = - e_{I_{d_0+1}}$. Hence, by our update rule (and the lack of projections which we proved at the beginning):
\[ w_{T_{d_0,d_1,1}}(I_t) = \alpha\eta \begin{cases}
    (d_0+2-t) & t\le d_1\\
    (d_0+1-t) & d_0< t\le  d_0\\
    1 & t=d_0+1\\
    0 & \textrm{o.w.}
\end{cases}\]
Which satisfies \cref{eq:sksd}. Now assume that \cref{eq:sksd} holds for $d_2$, and take $d_2+1 < d_0+2-d_1$ (otherwise, we are done). Notice that $T_{d_0,d_1,d_2+1} = T_{d_0,d_1,d_2}+1$. Observe that $w_{T_0,d_1,d_2}(d_0+2-d_2) =w_{T_0,d_1,d_2}(d_0+1-d_2)$ (notice that $d_0+1-d_2>d_1$), and our update rule is such that $\OS{t}= e_{I_{d_0+2-d_2}}- e_{I_{d_0+1-d_2}}$ and we obtain then:
\[ w_{T_{d_0,d_1,d_2+1}} =w_{T_{d_0,d_1,d_2}} - \eta \alpha e_{I_{d_0+2-d_2}} + \eta \alpha e_{I_{d_0+1-d_2}}
  = \alpha\eta \begin{cases}
    (d_0+2-t) & t \le  d_1\\
    (d_0 +1 -t) & d_1 < t < d_0 + 2 - ( d_2+1)\\
    (d_0+2 -t) & t= d_0 + 2 - (d_2+1) \\
    (d_0+1-t) & d_0+2-(d_2+1) < t\le  d_0\\
    0 & \textrm{o.w.}
\end{cases}\]
The most inner induction step is now complete. We now notice that $T_{d_0,d_1,d_0+1-d_1}= T_{d_0,d_1+1}$, which proves the middle-induction step. And we notice that $T_{d_0,d_0+1}= T_{d_0+1}$, which proves the whole induction argument.
\subsection{Proof of \cref{eq:decreasing}}
We only need to show that the following quantity is increasing in $t$

\[X_t= \max\left\{\sum_{i\in I_B} w_t^\SS(i): I_B\subseteq I, |I| = B\right\}.\]

But, as shown in \cref{prf:nemonster}, the update rule is such that we don't perform projections. It then follows easily from our update step. Indeed if we increase a set of coordinates by $\alpha \eta$, then clearly the $X_t$ only increases. Also, if we perform update of the form:
\[ w_t^\SS(I_j)=w_{t-1}^\SS(I_{j})+\alpha \eta,\quad  w_t^\SS(I_{j+1})=w_{t}^\SS(I_{j+1})-\alpha \eta,\]
for two consecutive and equal coordinates, then for any $I_B$: If $I_B$ includes same number of coordinates from $I_j$ as in $I_{j+1}$ then the magnitude doesn't change. If $I_B$ contains more $I_{j}$ then it increases, and if $I_B$ contains more from $i_{j+1}$ then consider $I_{B'}$ that swaps coordinates in $I_{j}$ with $I_{j+1}$ then we clearly have:

\[ \sum_{i\in I_{B'}} w_{t}^\SS(i)>\sum_{i\in I_{B'}} w_{t-1}^\SS(i) > \sum_{i\in I_{B}} w_t^\SS(i) 
\]

\section{Dimension independent lower bound for GD}\label{apx:oplow}
In this section we prove that the optimization error of GD in \cref{eq:gd_empirical} is optimal. The lower bound is an optimization error for first order methods and is well established (see \cite{bubeck2015convex}). The point here is to show that the bound is valid in any dimension, for GD.
\begin{claim}\label{cl:oplow}
    For every choice of $\eta, T$, there exists a convex and $1$ Lipschitz function $f(x):\mathbb{R}\to \mathbb{R}$ such that, if we run GD on $f$, then 
    \[ f(w^{GD}) - f(0) \ge \frac{\eta}{2} + \frac{1}{6\eta T}.\]
\end{claim}
\begin{proof}
    We divide the proof into two cases:

    \paragraph{case1: $\eta \ge 1/\eta T$}:

    In this case we choose $f(x)= |x-\gamma|$, where $\gamma>0$ is a arbitrarily small (may depend on $\eta$). It can be seen that for every even iteration, we have that:
    \[\nabla f(x_{2t}) = -1,\]
    at at every odd iteration 
    \[ \nabla f(x_{2t+1}) =+1.\]
    As such, for every even iteration we have that $x_{2t} = 0$, and at every odd iteration we have that $x_{2t+1}= \eta$. We thus have that $x^{GD}=\frac{\eta}{2}$ and
    \[ f(x^{GD})- 0 = \frac{\eta}{2} - 0 \ge \frac{\eta}{4} + \frac{1}{2\eta T}.\] 
 \paragraph{case1: $\eta \le 1/\eta T\le 1$}:
    In this case choose $f(x) = \alpha x$, where $\alpha = \frac{1}{(T+1)\eta}\le 1.$ Then, one can show that GD outputs 
    \[ x^{GD} = -\frac{\eta}{T}\sum t\cdot\alpha = -\frac{(T+1)\eta}{2} \alpha,\]
    and:

    \[f(x^{GD}) -f(-1) = \alpha - \frac{(T+1)\eta}{2} \alpha^2 \ge  \frac{1}{2(T+1)\eta} \ge \frac{1}{3T\eta} \ge \frac{1}{6T\eta}+\frac{\eta}{2}.\]
\end{proof}
    \end{document}

%% file: header.tex
\usepackage[utf8]{inputenc}
\usepackage[margin=1in]{geometry}

\usepackage{amsmath,amsfonts,amsthm,amssymb}
\usepackage{mathtools}

 \usepackage{mytimes}

\usepackage[numbers, sort&compress]{natbib}

\usepackage{thmtools}
\usepackage[hidelinks,colorlinks=true,linkcolor=blue,citecolor=blue]{hyperref}
\usepackage[capitalise]{cleveref}
\newcommand{\OS}[1]{{\mathcal O}^{(#1)}}
\newcommand{\OSS}{\mathcal{O}_{\mathcal{S}}}

\theoremstyle{plain}
\declaretheorem{theorem}
\declaretheorem[sibling=theorem]{lemma}
\declaretheorem[sibling=theorem]{corollary}
\declaretheorem{claim}

\declaretheorem[name=Open Question,sibling=theorem]{open}

\newtheorem*{theorem*}{Theorem}

\newcommand{\lbound}[1]{\Omega \left(\min\left\{#1,1\right\}\right)}
\newcommand{\ubound}[1]{O\left(#1\right)}

\declaretheoremstyle[
        spaceabove=\topsep, 
        spacebelow=\topsep, 
        bodyfont=\normalfont,
        notefont=\normalfont\bfseries,
        notebraces={}{},
        qed=$\blacksquare$, 
    ]{proofstyle}
\declaretheorem[style=proofstyle,numbered=no,name=Proof]{proof}

\usepackage{xcolor}
\usepackage[extdef=true]{delimset}
\usepackage{xspace}
\usepackage{crossreftools}

\crefname{claim}{Claim}{Claims}

\newcommand{\W}{\mathcal{W}}
\newcommand{\E}{\mathop{\mathbb{E}}}

\DeclarePairedDelimiter\floor{\lfloor}{\rfloor}

\newcommand{\wgd}{w_S^{\textrm{GD}}}

\renewcommand{\epsilon}{\varepsilon}

\newcommand{\f}[1]{f_{\scriptscriptstyle {(#1)}}}

\newcommand{\Z}{\mathcal{Z}}

\newcommand{\q}{q}

\newcommand{\labelthis}[1]{\stepcounter{equation}\tag{\theequation}\label{#1}}

\newcommand{\ignore}[1]{}

\newcommand{\V}{\mathcal{V}}
\renewcommand{\Z}{\mathcal{Z}}

\renewcommand{\O}{\mathcal{O}}

\renewcommand{\S}{\mathcal{S}}

\newcommand{\xz}[1]{x_#1}
\newcommand{\wz}[1]{w_#1}
\renewcommand{\u}[1]{\left(\wz{#1},\xz{#1}\right)}
\newcommand{\hz}[1]{\nabla h_z(\xz{#1})}
\newcommand{\oz}[1]{\O_{\SS_#1,t_#1,z}}
\newcommand{\g}[1]{\left(\oz{#1},\hz{#1}\right)}
\newcommand{\h}[1]{h_z(\xz{#1})}

\newcommand{\theboundproto}[1]{\frac{1}{#1\cdot 272\cdot 16^2} \cdot \min\left \{\frac{d}{1032m},1\right\}\cdot \min \left\{ \eta 
\sqrt{\min\{\floor{d^3/136},T\}}, 1\right\}}
\newcommand{\theboundpre}{\theboundproto{\sqrt{2}}}
\newcommand{\thebound}{\theboundproto{2}}
\renewcommand{\f}[1]{f(\wz{#1},z)}
\renewcommand{\ss}{\mathbf{s}}